\documentclass[lettersize,journal]{IEEEtran}

\usepackage{array}
\usepackage[caption=false,font=normalsize,labelfont=sf,textfont=sf]{subfig}
\usepackage{textcomp}
\usepackage{stfloats}
\usepackage{url}
\usepackage{verbatim}
\usepackage{cite}
\usepackage{amsmath}
\usepackage{amsfonts}
\usepackage{amssymb}
\usepackage{graphicx}
\usepackage{booktabs}
\usepackage{subeqnarray}
\usepackage{algorithmic}
\usepackage{algorithm}
\usepackage{amstext}
\usepackage{psfrag}
\usepackage{float}
\usepackage{multirow}
\usepackage{color}
\usepackage[colorlinks=true, allcolors=blue]{hyperref}
\usepackage{tabularx}
\usepackage[table]{xcolor}
\usepackage{booktabs}
\usepackage{comment}
\usepackage{enumitem}
\newtheorem{theorem}{Theorem}

\newtheorem{proof}{Proof}

\newif\ifcomments

\commentsfalse 

\ifcomments

\newcommand{\MB}[1]{\textcolor{red}{{\bf MB}: {#1}}}
\newcommand{\NP}[1]{\textcolor{magenta}{{\bf NP}: {#1}}}
\newcommand{\BT}[1]{\textcolor{violet}{{\bf BT}: {#1}}}
\newcommand{\NM}[1]{\textcolor{teal}{{\bf NM}: {#1}}}
\newcommand{\YG}[1]{\textcolor{purple}{{\bf YG}: {#1}}}
\else
\newcommand{\MB}[1]{}
\newcommand{\NP}[1]{}
\newcommand{\BT}[1]{}
\newcommand{\NM}[1]{}
\newcommand{\YG}[1]{}
\fi

\newcommand{\veps}{\mathbb{\varepsilon}}
\newcommand{\XX}{\mathcal{X}}
\newcommand{\YY}{\mathcal{Y}}
\newcommand{\BB}{\mathcal{B}}
\newcommand{\CC}{\mathcal{C}}
\newcommand{\DD}{\mathcal{D}}
\newcommand{\EE}{\mathcal{E}}
\newcommand{\FF}{\mathcal{F}}
\newcommand{\NN}{\mathcal{N}}

\begin{document}

\title{The Role of Input Dimensionality in the Emergence and Targeted Control of Adversarial Examples}

\author{Nasrin Malekzadeh Goradel, 
Niccolò Pancino, ~\IEEEmembership{Member,~IEEE,} 
Yaser Gholizade Atani, 
Benedetta Tondi,~\IEEEmembership{Member,~IEEE,} 
Giovanni Bellettini, 
Mauro Barni,~\IEEEmembership{Fellow,~IEEE}
\thanks{\noindent \textbf{Acknowledgments.} The research by N. M. Goradel and Prof. M. Barni was funded by a grant provided by Leonardo SpA. The work was also  partially supported by the EU - NextGenerationEU, under the National Recovery and Resilience Plan (NRRP) - Extended Partnership SERICS, Spoke 3 Attacks and Defences, Mission 4, Component 2, Investment 1.3, AI-RESCUE.  PE00000014, CUP: B63C24000490006.}
}

\markboth{Journal of \LaTeX\ Class Files,~Vol.~14, No.~8, August~2021}%
{Shell \MakeLowercase{\textit{et al.}}: A Sample Article Using IEEEtran.cls for IEEE Journals}

\IEEEpubid{0000--0000/00\$00.00~\copyright~2021 IEEE}

\maketitle

\begin{abstract}
Several theoretical works have tried to explain the adversarial vulnerability of deep neural networks through properties of high-dimensional geometry. However, the assumptions underlying these works are rarely examined empirically, and systematic evidence remains limited. In this work, we present a systematic study of the role of input dimensionality in both the emergence and the targeted control of adversarial examples. We first analyse the scope and limitations of existing theoretical frameworks based on concentration of measure, showing that real image classes exhibit strong empirical localization, beyond what such theories typically assume. We then conduct an extensive empirical evaluation across hierarchical image datasets spanning a wide range of input dimensionalities and diverse neural architectures. Our results consistently show that adversarial examples become easier to construct as dimensionality increases.
We also investigate how input dimensionality affects the additional difficulty of crafting targeted adversarial examples. In particular, we provide theoretical arguments showing that high-dimensional geometry implies that enforcing a specific target label entails only a limited additional distortion compared to untargeted attacks. We corroborate this insight through extensive experiments, demonstrating that the gap between targeted and untargeted perturbations remains small and further narrows as input dimensionality increases. While, taken together, our findings establish high input dimensionality as a fundamental factor underlying the emergence and targeted control of adversarial examples, whether this phenomenon primarily arises from the interplay between high-dimensional geometry and data distributions or from the architectural properties of deep neural networks remains an open question.
\end{abstract}


\section{Introduction}
\IEEEPARstart{S}{ince} their existence was first pointed out \cite{szegedy2013intriguing}, adversarial examples - small, carefully crafted perturbations that cause machine learning models to misclassify inputs - have become the subject of intense research. A commonly proposed explanation for the existence of adversarial examples points to the concentration of measure phenomenon \cite{Ledoux1}, with several works suggesting that adversarial examples are an inherent consequence of high-dimensional geometry \cite{fawzi2018analysis, shafahiadversarial, mahloujifar2019curse, gilmer2018advsph, gilmer2019adversarial, dohmatob2019generalized, de2021adversarial}.
The core intuition behind these works is that, due to the concentration of measure, as the dimensionality of the input space increases, most data points become arbitrarily close either to class decision boundaries \cite{shafahiadversarial} or to misclassification regions \cite{fawzi2018analysis,mahloujifar2019curse}. 
While these theoretical analyses are compelling due to their elegance and generality, they rely on assumptions about the concentration properties of data distributions whose validity in practical settings remains largely unverified. Some works have tried to clarify this issue by characterizing the conditions under which adversarial examples may be avoidable. In particular, \cite{pal2023avoidable} shows that when class-conditional distributions are strongly localized relatively to inter-class separation, adversarial robustness is in principle achievable.
Whether the conditions stated in \cite{pal2023avoidable} are satisfied by real-world data, and how they interact with input dimensionality in practical learning scenarios, however, remains unclear.

Motivated by the above arguments, in this work we take a systematic empirical perspective on the role of input dimensionality in the emergence of adversarial examples.
To start with, we introduce a pool of hierarchical datasets composed of images at multiple resolutions, obtained by downsampling high-resolution images to progressively smaller sizes. This allows us to vary input dimensionality in a controlled manner while preserving semantic content. Using these datasets, we first show that as input dimensionality increases, class-conditional distributions become increasingly localized, revealing a geometric regime that departs significantly from the assumptions underlying most concentration-based theoretical analyses.
Then, we perform an extensive experimental evaluation across diverse neural architectures, including both standard and robustness-enhanced models, and consistently observe that adversarial examples become easier to construct as input dimensionality increases.

In the second part of the paper, we move beyond untargeted adversarial examples and investigate targeted attacks, in which the adversary aims to induce misclassification into a specific target class. Targeted attacks are commonly regarded as more difficult to construct due to their stronger control requirements, in addition, most theoretical analyses focus on the untargeted setting.
As a first contribution, we extend existing geometric arguments for untargeted adversarial examples to the targeted setting only. In particular, relying on a classical extension of the concentration function to sets of arbitrary size \cite{Ledoux1}, we show that, under mild assumptions on the number of classes and the volume occupied by the decision regions defined by the attacked classifier, the additional difficulty of crafting targeted adversarial examples remains inherently limited and becomes asymptotically negligible as dimensionality grows. 
Then, to assess the validity of the theoretical predictions - which, like for untargeted attacks, rely on assumptions whose plausibility in real-world settings is questionable - we conduct a systematic empirical study across the hierarchical datasets introduced earlier. Our results consistently reveal that: i) targeted attacks also become easier to craft as dimensionality increases, and ii) the additional cost required to control the target class of the attack remains limited and diminishes with input size.

With the above ideas in mind, the main contributions of this paper can be summarised as follows:

\begin{itemize}[leftmargin=*]
\item We introduce a pool of hierarchical image datasets that enable a controlled study of adversarial vulnerability as a function of input dimensionality.
\item We show that as input dimensionality increases, the support of class-distributions shrinks  at an exponential rate, thus challenging the key assumptions of most theoretical analysis relying on the concentration of measure phenomenon.
\item We empirically demonstrate that, despite strong mass localization of class distributions, adversarial examples become progressively easier to craft as input dimensionality increases, with the minimum distortion required to achieve a fixed attack success rate decreasing systematically when resolution increases.
\item We extend existing geometric concentration arguments to the case of targeted adversarial attacks.
\item We empirically show that the greater control afforded by targeted attacks comes at only a modest additional cost, and that this cost decreases with input dimensionality.
\end{itemize}

Overall, our analysis provides empirical evidence for a dimensionality-driven effect akin to a curse of dimensionality in the emergence and targeted control of adversarial examples. At the same time, the empirical violation of key assumptions underlying existing theoretical explanations suggests that high-dimensional geometry alone may be insufficient to fully explain adversarial vulnerability, leaving open the question of how geometry, real data distributions, and classifier architecture jointly contribute to this phenomenon.

\section{Prior Art}

The first attempt to use high-dimensional geometry to explain the emergence of adversarial examples can be traced back to the work of Gilmer et al.~\cite{gilmer2018advsph}. In that paper, geometric properties of high-dimensional spaces are exploited to show that, when data are distributed on concentric high-dimensional spheres, adversarial examples exist even for an optimal classifier. Similar arguments were later employed by Fawzi et al.~\cite{FFF18} in a much more general setting, showing that for artificially generated images obtained from a Gaussian distribution in a latent space and a smooth mapping between latent and image spaces, adversarial examples necessarily emerge for \emph{any} classifier.

While concentration phenomena are already implicit in~\cite{FFF18}, the central role of concentration of measure in adversarial vulnerability is made explicit in~\cite{diochnos2018adversarial, mahloujifar2019curse}. A further merit of~\cite{diochnos2018adversarial} is the systematic analysis of several equivalent definitions of adversarial examples, distinguishing between the case in which the attacker aims at pushing a sample into the \emph{error region} of the classifier (ER), inducing a \emph{change of prediction} (CP), or crafting a sample that is misclassified with respect to the true label of the original input (\emph{Corrupted Instance}, CI).

Adopting the ER definition Mahloujifar et al.~\cite{mahloujifar2019curse} prove that adversarial examples are inevitable for any non-ideal classifier and for any data distribution exhibiting sufficiently strong concentration properties, including L\'evy families and the uniform distribution over the unit hypercube.
The results in ~\cite{mahloujifar2019curse} are further extended in~\cite{dohmatob2019generalized}, where they are used to analyze the trade-off between standard and adversarial accuracy.

The analyses in~\cite{mahloujifar2019curse} and~\cite{dohmatob2019generalized} rely on the ER definition and therefore apply only to non-ideal classifiers, i.e., classifiers with non-zero error probability in the absence of attacks. This limitation is removed by Shafahi et al.~\cite{shafahiadversarial}, who show that, for high-dimensional data, adversarial examples are inevitable even for \emph{perfect} classifiers, provided that inputs are distributed over the unit hypercube and that the class-conditional densities are uniformly bounded.

A natural question stemming from the analyses reviewed above is whether natural images actually satisfy the concentration properties assumed by theoretical works. In~\cite{shafahiadversarial}, the boundedness assumption on the input distribution is investigated in a simplified setting in which, starting from small images such as those in MNIST, the input dimensionality is artificially increased by duplicating each pixel into a \(b \times b\) block of identical values.
The findings in~\cite{shafahiadversarial}, however, are not conclusive, primarily due to the simplicity of the considered setting.
More broadly, several recent works have attempted to empirically assess the concentration properties of real-world datasets, including~\cite{mahloujifar2019empirically, prescott2021improved, zhang2022incorporating}, where the concentration function of MNIST and CIFAR-10 datasets is estimated and used to validate the predictions of theory. While these studies establish an important bridge between theoretical insights and empirical data, they do not directly address the role of input dimensionality in the emergence of adversarial examples for two main reasons. First, the analysis is restricted to datasets composed of low-resolution images only. Second, and more importantly, no effort is made to investigate how adversarial vulnerability evolves as the image resolution increases \emph{within the same visual domain}, and hence how concentration itself scales with dimensionality under fixed semantic content.

The question of whether the distributional assumptions underlying theoretical results are satisfied by natural images becomes even more relevant in light of recent work showing that adversarial examples may, in principle, be avoidable when the data distribution satisfies sufficiently strong geometric constraints. In particular, Pal et al.~\cite{pal2023avoidable} show that adversarial examples are not inevitable if the class-distribution of the underlying data is sufficiently localized. They prove that a necessary condition for adversarial examples to be avoidable is that the volume of the support of class-distributions decay exponentially fast with the ambient dimension, at a rate depending on the admissible perturbation. They also derive a sufficient condition to ensure avoidability
Notably, these requirements are considerably stronger than those typically expected to hold for natural images, thereby sharpening the following central question: {\em as image resolution increases within a fixed domain, do real-world data distributions evolve toward a regime of stronger concentration and inevitability of adversarial examples, or toward one in which adversarial examples might, in principle, be avoidable?} The first goal of this paper is to answer this question empirically by relying on a comprehensive experimental campaign
explicitly controlling the effect of input dimensionality through resolution-scaled datasets sharing the same semantic content.

Most theoretical and practical works on adversarial examples focus on the untargeted setting. Extending concentration-based existence results to targeted attacks is not immediate, since the target region of the attack may occupy only a small fraction of the input space, preventing the direct application of basic concentration arguments.
As a consequence, the additional distortion required to enforce a prescribed target class - and, crucially, how such a cost scales with input dimensionality - remains largely unexplored. Even the definitions introduced in \cite{diochnos2018adversarial}, as well as the metrics commonly used to quantify adversarial robustness, require some care in the targeted setting, for instance by distinguishing between average distortions over all source-target class pairs and worst-case distortions associated with the most difficult class transitions.

From a practical perspective, several works have shown that targeted adversarial attacks can also be crafted effectively, albeit typically requiring larger perturbations than their untargeted counterparts \cite{CarliniWagner2017,CroceHein2020,Tolias2019,Dong18,akhtar2018threat}. However, a systematic investigation of how input dimensionality affects the ease of constructing targeted adversarial examples is still missing. 
This observation naturally leads to the second research question addressed in this paper: \emph{As input dimensionality increases within a fixed visual domain, how does the additional distortion required to enforce a specific target class evolve?} In particular, does the gap between targeted and untargeted attacks widen, remain stable, or shrink as image resolution grows? The second part of our work is devoted to answering this question, both from a theoretical and an empirical standpoint.

\section{Preliminary Notions, Definitions and Research Questions}
\label{sec.notanddef}

In this section we introduce the mathematical notation used throughout the paper and provide a rigorous formulation of adversarial examples and adversarial risk. We then review the key concepts related to the concentration of measure phenomenon, and formulate the goals of this work.

Let $\XX$ be the space of the to-be-classified instances $x \in \XX$, and $\YY=\{1, \dots, m\}$ be 
the set of labels. 
In the following, with a slight abuse of notation, we indicate with $\mu$ the probability distribution of the instances in $\XX$ and the probability measure induced by such a distribution on the set $\XX$, the exact meaning being always clear from the context.
The association of instances and labels is defined by a ground truth function $c:\XX \rightarrow \YY$ \footnote{We assume that each $x \in \XX$ is linked to a unique, deterministic, ground truth label. This label might reflect, for example, a human's classification of an image. This perspective differs from probabilistic approaches, where the relationship between samples and labels is modeled by a probability distribution over $\XX \times \YY$.}. 
Given a classifier $f: \XX \rightarrow \YY$, a common way to define an adversarial example is as a perturbed sample $x' = x +\delta \in \XX$ such that $f(x') \ne f(x)$ \cite{szegedy2013intriguing}, where $\delta$ denotes an imperceptible perturbation constrained by a predefined norm.
Adversarial examples following this definition are sometimes referred to as {\em Prediction Change} (PC) adversarial examples \cite{diochnos2018adversarial}.  A limitation of the PC definition is that it does not account for the ground truth function $c$. 
For this reason, an alternative definition requires that $f(x') \ne c(x')$. Adversarial examples defined in this way are referred to as {\em Error Region} (ER) adversarial examples \cite{diochnos2018adversarial}. It is immediate to see that this second definition makes sense only for non-ideal classifiers. For an ideal classifier, in fact, $f(x) = c(x)$ for every $x$ and the conditions $f(x') \ne c(x')$ cannot be satisfied.
In \cite{diochnos2018adversarial}, a third definition of adversarial examples is introduced, namely {\em Corrupted Instance} (CI) adversarial examples, for which it is required that $f(x') \ne c(x)$. In this case, adversarial examples may also exist for an ideal classifier. In fact, in such a case, the CI definition is equivalent to the PC one.

In practical scenarios, where classifiers are not perfect  but achieve high accuracy, we often have $f(x) = c(x)$. In addition, the so called proximity assumption approximately holds, according to which, for small perturbations we (almost) always have $c(x) = c(x')$. As a result, the three definitions outlined above tend to be nearly equivalent.

With regard to the metric used to constraint the maximum allowed perturbation, most works adopt either the $\ell_2$ or the $\ell_\infty$ norm. In this work, we focus on the $\ell_2$ case.

With the above definitions, we can define the adversarial risk as the probability over the distribution of input samples that an adversarial example exists at a given distance from the unperturbed input, formally:
\begin{equation}
\texttt{Risk}_\veps = \text{Pr}\{x' \in \BB_\veps(x)\cap \XX~ s.t.~ f(x') \ne \{f(x), c(x'), c(x)\} \},
\label{eq.advrisk}
\end{equation}

where the probability is taken over the input sample distribution $\mu$, $\BB_\veps(x)$ indicates the ball of radius $\veps$ centered in $x$, and where the exact condition imposed on $f(x')$ depends on the definition of adversarial examples (PC, ER, or CI).

\subsection{Concentration of measure phenomenon}
\label{sec.concentration}
The concentration of measure phenomenon refers to a property of some probability measures in metric spaces (especially high-dimensional ones) according to which most of the probability is arbitrarily close to every large enough set. More specifically, let $(\XX, d, \mu)$ be a metric measure space, where $\XX$ is a space of dimension $n$, $d$ is a distance induced by a metric defined on $\XX$, and $\mu$ is a Borel probability measure on $\XX$.
The formalisation of the concentration of measure phenomenon passes through the definition of the concentration function $\alpha: [0, \infty) \rightarrow [0,1]$ \cite{Ledoux1}:
\begin{equation}
\alpha(\veps) := \sup \left\{ 1 - \mu(A_\veps) : A \subset \XX, \mu(A) \ge 1/2  \right\},
\label{eq.confun}
\end{equation}
where $A_\veps := \{ x \in \XX: d(x, A) \le \veps\} $ indicates the $\veps$-neighbourhood (or $\veps$-expansion) of $A$. It turns out that in many high-dimensional spaces, like high-dimensional spheres with uniform measure, and product spaces, $\alpha(\veps)$ decreases exponentially fast with $\veps$.
 
\subsection{Adversarial examples and concentration of measure} 
\label{sec.adv_and_conc}

Several researchers have used the concentration of measure phenomenon to explain the pervasive existence of adversarial examples, demonstrating that, as the dimensionality of input samples increases\footnote{This is the case, for instance, with colour digital images, where the input sample dimensionality corresponds to three times the number of pixels, often reaching tens or even hundreds of thousands of dimensions.} adversarial examples can always be constructed even when $\veps$ is arbitrarily small.

By adopting the PC definition of adversarial examples, Shafahi et al. \cite{shafahiadversarial} prove that given a classification problem with $m \ge 2$ classes, each distributed over the unit hypercube $[0, 1]^n$ with density functions $\{ \mu_i\}_{i=1}^m$, for any $f$ that partitions the hypercube into disjoint measurable subsets, given a sample $x$ belonging to class $i$, one of the following holds:
\begin{itemize}
    \item $x$ is misclassified by $f$
    \item $x$ has an adversarial example $x'$, with $\|x- x'\|_2 \le \veps$
\end{itemize}
with probability\footnote{The inequality below is derived directly from the proof of Theorem 2 in \cite{shafahiadversarial} (appendix A), and corrects a typo in the statement of the theorem.}
\begin{equation}
    P_{e,adv} \ge 1-U_i \frac{e^{-\pi \varepsilon^2}}{2 \pi \varepsilon}
\label{eq.asrShafahi}  
\end{equation}
where $U_i$ is the supremum of $\mu_i$.
The only condition behind this result is  that the fraction $\FF_i$ of the hypercube assigned to class $i$ by $f$ 
has a volume smaller than 1/2, a condition that is reasonably met by any classifier with more than 2 classes. When the input sample $x$ corresponds to an image, it is convenient to express the magnitude $\varepsilon$ of the perturbation in terms of mean square error $\gamma$, that is:
\begin{equation}
	\frac{\veps^2}{n} = \gamma,
\label{eq.boundmse}  
\end{equation}
allowing us to reformulate \eqref{eq.asrShafahi} as:
\begin{equation}
	P_{e,adv} \ge 1-U_i \frac{e^{-\pi n \gamma}}{2 \pi \sqrt{n \gamma}}.  
\label{eq.asrmse}  
\end{equation}
If $U_i$ is bounded, then, as $n$ increases, the probability that $x$ admits an adversarial example (or it is misclassified outright) approaches 1.
The intuition behind the arguments used in \cite{shafahiadversarial} is that due to the concentration of measure phenomenon, for any set occupying less than half of the input space, the distance to the set boundary of most of the points is such that they can be pushed outside the set by introducing an arbitrarily small mean square error.

Another line of research \cite{mahloujifar2019curse, diochnos2018adversarial, dohmatob2019generalized} establishes the inevitability of adversarial examples under the ER definition by introducing a non-empty error region $\EE$ where $f(x)\neq c(x)$. In this setting, if $(\XX,d,\mu)$ is concentrated, a high adversarial risk can be achieved with vanishing mean square distortion as the input dimensionality increases. The intuition is that most correctly classified samples lie close to the boundary of $\EE$ and can therefore be pushed into the error region by increasingly small perturbations.

\subsection{Limits of theory and goals of the paper}
\label{sub.limits}

A natural question is whether the distributional assumptions underlying concentration-based inevitability results are satisfied in practice. All these results rely, in one form or another, on the assumption that the data to be classified live in sufficiently regular metric spaces, and that the probability mass of each class is smoothly spread over these spaces, so that the corresponding concentration function decays rapidly with the input dimensionality~$n$.
By contrast, if most of the probability mass associated with one or more classes is confined to small regions, and if such regions are sufficiently well separated, one may reasonably expect the existence of robust classifiers. This intuition is formalized in~\cite{pal2023avoidable} through the introduction of $(a,b)$-localized and $(a,b,c)$-strongly localized distributions%
\footnote{In~\cite{pal2023avoidable} the terms \emph{concentrated} and \emph{strongly concentrated} distributions are used. To avoid any ambiguity with metric concentration phenomena in high-dimensional spaces, we instead adopt the terms \emph{localized} and \emph{strongly localized} distributions.}. 
Specifically, a distribution~$\mu$ over ~$\XX$ is said to be $(a,b)$-localized if there exists a subset~$S \subset \XX$ such that $\mu(S) \ge 1 - b$ and $\mathrm{Vol}(S) \le C e^{-a n}$ for some constant~$C$, where $\mathrm{Vol}(\cdot)$ denotes a volume measure over~$\XX$. 
Theorem~2.1 in~\cite{pal2023avoidable} states that a necessary condition for the existence of a classifier satisfying $\texttt{Risk}_a \le 1 - b$ is that $\mu_i$ is $(a,b)$-localized for at least one class~$i$.
Localization alone does not guarantee robustness; however, it invalidates the assumptions under which concentration-based inevitability results are derived. This limitation is particularly evident in Shafahi et al.~\cite{shafahiadversarial}. From the definition of an $(a,b)$-localized distribution, it follows that the supremum of~$\mu_i$ - denoted as~$U_i$ in~\eqref{eq.asrmse} - grows exponentially with~$n$, causing the bound in~\eqref{eq.asrmse} to become vacuous for appropriate choices of~$\gamma$, $a$, and~$b$.

The first goal of this paper stems from the observation that, when localization conditions are met, the assumptions underpinning concentration-based impossibility results are violated and theory alone becomes inconclusive with respect to the inevitability of adversarial examples. This motivates a direct empirical investigation aimed at assessing whether the distribution of natural images exhibits mass localization.

Localization of data distribution constitutes only a necessary condition for the existence of a robust classifier, as it does not account for the separation between the distributions of different classes. This limitation motivates the notion of \emph{strongly localized} distributions, which additionally require class separation and provide a sufficient condition for robustness~\cite{pal2023avoidable}). In practice, however, assessing whether the class-conditional distributions of natural images satisfy strong localization properties is generally an intractable problem for real-world image distributions. This observation leads to the second goal goal of our work: to evaluate the actual impact of input dimensionality on the emergence of adversarial examples in practical settings.

\section{Experimental Setting}
\label{sec::experimental_setup}

In this section, we describe the experimental setting used throughout our work. 

\subsection{Hierarchical datasets}
\label{subsec:datasets}
As we said, most empirical studies on adversarial robustness and concentration-based theories rely on very low-resolution benchmarks such as MNIST ($28 \times 28$) and CIFAR-10 ($32 \times 32$), while investigations across input dimensions often compare datasets with different semantic content, making it difficult to isolate the effect of dimensionality, or consider artificial scaling as done with MNIST in~\cite{shafahiadversarial}.
To overcome these limitations, we built three hierarchical datasets containing images with identical semantic content at different resolutions. Rather than starting from small images and artificially constructing higher-resolution inputs, as done in~\cite{shafahiadversarial} for MNIST via pixel duplication, we started from high-resolution images and generated lower-resolution versions by downsampling. 

We selected datasets characterized by a large native resolution enabling consistent downsampling from the highest resolution to the lowest. Specifically, starting from the images at original resolution, we built downscaled versions of each dataset at $64 \times 64$, $128 \times 128$, $256 \times 256$ and $400 \times 400$ resolutions\footnote{In the rest of the paper we refer to these dimensions as resolutions 64, 128, 256, and 400.}, by using bilinear interpolation to resize the images to the target size, and scaling pixels by a constant factor of 255 so to map their value in the [0, 1] range.

\paragraph{VegSeed}
The VegSeed dataset was derived from the original VegSeedsBD dataset~\cite{vegseedsbd}, which comprises 4,500 high-resolution JPEG images (3,456 $\times$ 4,608 pixels) evenly distributed across 15 vegetable seed categories - each containing 300 images - and used for single-label, multiclass image classification.
In its original form, the data was partitioned into two distinct subsets based on the capture method: individual seeds and bulk samples. 
In our version, we merged these subsets, removing the structural distinction between single and bulk images while maintaining the original 15 classes, distribution, and task.

\paragraph{Food-30}
The Food-30 dataset was derived from Food-101~\cite{food101}, a widely used benchmark in computer vision designed for single-label multiclass food image classification. Food-101 contains 101,000 real-world images divided evenly across 101 food categories, with 1000 images per class at varying resolutions, and a predefined split of 750 training images and 250 test images per class.
From the original dataset, we selected 30 classes and sampled 500 images per class, for a total of 15,000 images with resolution 512 $\times$ 512.

\paragraph{ResynthDB}
We built this dataset starting from the Resynthesis Dataset\footnote{The dataset is publicly available at \url{https://www.kaggle.com/datasets/pietrob92/resynthesis-dataset?select=resynthesis_dataset}.} ~\cite{sourceattribution} 
, which contains images generated by 10 generators (\textit{Bing, Firefly, Flux-dev, Freepik, Imagen, Leonardoai, Midjourney, Nightcafé, Stabilityai,} and \textit{Starryai}) each acting as a class in a single-label, multiclass prediction task.
Each generative model was provided with 100 textual prompts to produce the image set. While most classes consist of 1,100 samples, the Bing and Firefly categories contain 4,247 and 4,364 images, respectively.
The resulting dataset contains 17,411 images with resolution 1024 $\times$ 1024.

Table \ref{tab:datasets} summarizes the main characteristics of the datasets we have built. 
\begin{figure}
    \centering
    \includegraphics[width=\columnwidth]{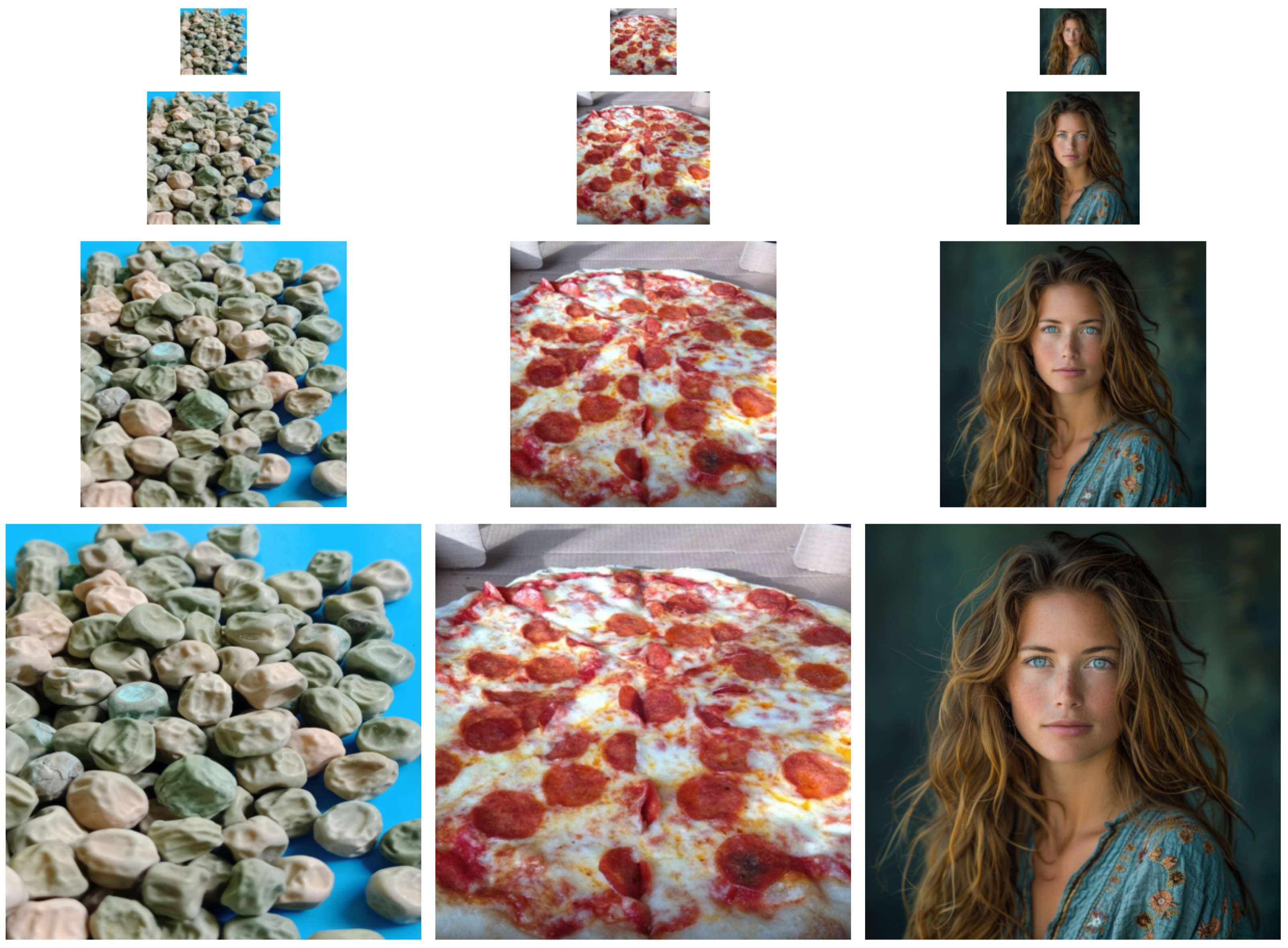}
    \caption{Examples of images from VegSeed, Food-30, and ResynthDB datasets (columns, from the left) at the considered input resolutions 64, 128, 256, and 400 (rows, from the top).}
    \label{fig::samples_from_datasets}
\end{figure}
\vspace{-0.1cm}
An example of the images contained in the hierarchical datasets we have built is provided in Fig. \ref{fig::samples_from_datasets}.
\begin{table}
\centering
\caption{Hierarchical datasets description.}
\label{tab:datasets}
\begin{tabular}{lccc}
\toprule
    & \textbf{VegSeed} & \textbf{Food-30} & \textbf{ResynthDB} \\
    \midrule
    \textbf{Classes} & 15 & 30 & 10 \\ 
    \textbf{No. samples} & 4,500 & 15,000 & 17,411 \\
    \textbf{Samples per class} & 300 & 500 & variable\\
    \textbf{Balanced} & yes & yes & no \\
    \textbf{Original resolution} & 3,456$\times$4,608 & 512$\times$512 & 1,024$\times$1,024 \\
\bottomrule
\end{tabular}
\end{table}
\vspace{-0.1cm}

All datasets were split into training, validation, and test sets, with a ratio of 80\%, 10\%, and 10\% of the samples, respectively. The same split was applied consistently across all resolutions. 
To account for class imbalance, we used class-weighted loss during training. 

\subsection{Architectures}
\label{sec:allmodels}

Since the theoretical results discussed in the previous sections are stated for \emph{any} classifier, it is crucial that the empirical analysis is carried out across a broad and heterogeneous set of models. For this reason, our experimental campaign includes multiple state-of-the-art classification backbones as well as several robust strategies representative of modern defenses against adversarial examples. This allows us to assess whether the observed trends with input dimensionality persist not only for a few standard models, but also for classifiers explicitly optimized for adversarial robustness.

With the above considerations in mind, we evaluated four backbone architectures: MobileNetV3-Small~\cite{howard2019searching} (hereafter \emph{MobileNet}), EfficientNetV2-Small~\cite{tan2019efficientnet} (hereafter \emph{EfficientNet}), ResNet-50~\cite{he2016deep}, and a hybrid CLIP+MLP model, in which a MultiLayer Perceptron (MLP) classifier is paired with a CLIP~\cite{radford2021learning} Vision Transformer backbone for feature extraction.
All CNN-based models were pre-trained on the ImageNet dataset~\cite{imagenet} and then fine-tuned on the hierarchical datasets.
In contrast, the CLIP + MLP architecture consists of a frozen pre-trained CLIP image encoder from OpenAI - originally trained on the WebImageText (WIT) dataset~\cite{radford2021learning} - followed by a randomly initialized MLP for the final classification task. 
In this case, the dataset construction 
is followed by an additional procedure based on the pre-trained CLIP parameters, which consists of a normalization with specific mean and standard deviation values, and a resizing transformation into a $224 \times 224$ pixels image (by means of a bilinear interpolation), to ensure the input images match the specifications of the CLIP vision encoder.

For each architecture, we first trained a \emph{Standard} model using only clean, unperturbed images. We then used these models as starting points to apply state-of-the-art defenses.
As a first defense, we trained a \emph{Robust} model via adversarial training.
The Robust model is initialized from the corresponding Standard model and subsequently fine-tuned using both clean and adversarially perturbed data generated via a PGD$_{\ell_2}$ attack~\cite{madry2017towards}.
In the experiments, the attack is performed with large maximum allowed perturbation $\veps$ and small step size  $\alpha$, and the optimization is stopped as soon as a misclassification is induced at an iteration $s' \leq s$, so as to obtain the minimum distortion required to deceive the model.\footnote{Note that, by adopting a large $\veps$ as we are doing, the projection becomes ineffective and PGD boils down to a standard iterative gradient ascent attack (PGD without projection). For simplicity, in the following, we refer to the attack as PGD with a slight abuse of terminology.}
In the training procedure, we set $\veps=1000, s=200, \alpha=0.01$, to balance between positive and negative training examples. 
In the evaluation procedure we used a higher $s=5000$ and also optimized $\alpha$ to ensure the attack reached a predefined efficacy threshold while maintaining minimal perturbation overhead.

In addition to adversarial training, we evaluated two defense mechanisms based on input pre-processing, namely JPEG compression and AI-based denoising~\cite{zuo2018convolutional}. In this case, models are trained on pre-processed clean images. At test time, inputs are pre-processed before being fed to the classifier, under the assumption that the pre-processing step can remove the subtle perturbations introduced by adversarial attacks.

Overall, our experimental campaign spans 4 backbone architectures, 4 training/defense regimes, and 4 input resolutions, resulting in a total of 
64 trained models per dataset. Across the three datasets considered in this work, this amounts to 192 trained networks.
Note that each model was trained and evaluated using images at a single resolution only, ensuring that no model was exposed to multiple resolutions during training or testing.

\subsection{Metrics}
The robustness of the various models and the effectiveness of the attacks have been evaluated by considering three complementary metrics characterizing the performance on clean data and robustness under adversarial perturbations. Let $\DD = \{x_i\}_{i=1}^N$ be the evaluation dataset, where $x_i \in \left[0,1 \right]^n$ is an input sample. Let $c, f$ defined as in Section \ref{sec.notanddef}.
 
\paragraph{Standard Accuracy (SAcc)}
The Standard accuracy measures the fraction of samples in $\DD$ that are correctly classified in the absence of attacks:
\begin{equation}
\label{eq::sacc}
\mathrm{SAcc} = \frac{1}{N}\sum_{i=1}^{N}\mathbf{1}\left[ f(x_i) = c(x_i) \right].
\end{equation}
where $\mathbf{1}$ denotes the usual indicator function. 

\paragraph{Attack Success Rate (ASR)}
To disentangle the attack effectiveness from the classification errors in the absence of attacks, we restrict the evaluation of the attack success rate to the subset of samples that are correctly classified on clean data, that is, samples belonging to the set $\DD_{corr} = \left\{ x_i : f(x_i) = c(x_i) \right\}$, with cardinality $|\DD_{corr}| = N \cdot \mathrm{SAcc}$. \\

By indicating with $x_i'$ the result of an attack, the Attack Success Rate (ASR) is defined as the proportion of initially correct samples that are misclassified after the attack, that is:
\begin{equation}
\label{eq::asr}
\mathrm{ASR}
= \frac{1}{|\DD_{corr}|} \sum_{x_i \in \DD_{\text{corr}}} \mathbf{1}\left[ f(x_i') \neq c(x_i) \right].
\end{equation}

\paragraph{Adversarial Accuracy (AAcc)}
Adversarial accuracy measures the fraction of samples in $\DD$ that are correctly classified after the attack, where the attack is applied only to the subset 
$D_{corr}$ of correctly classified samples:
\begin{equation}
\label{eq::aacc}
\mathrm{AAcc} = \frac{1}{N} \sum_{x_i \in \DD_{corr}} \mathbf{1}\left[ f(x_i') = c(x_i) \right].
\end{equation}

By construction, Eq. (\ref{eq::asr}) quantifies the fraction of correctly classified samples whose predictions are altered by the attack, highlighting how vulnerable is the model, while Eq. (\ref{eq::aacc}) measures the overall fraction of samples correctly classified after the attack.
By referring to the definitions given in Section \ref{sec.notanddef}, AAcc is the empirical counterpart of $1-\texttt{Risk}_\veps$ when the CI definition of adversarial examples is adopted and the attack perturbation is bounded by $\veps$.
It is also easy to see that:
\begin{equation}
\label{eq::aac_asr}
\mathrm{AAcc} = \mathrm{SAcc} \cdot (1 - \mathrm{ASR}),
\end{equation}

highlighting how robustness in the presence of attacks depends jointly on the model’s accuracy on clean samples and its robustness against adversarial perturbations.

\paragraph{Distortion (MSE)}
To compare a generic clean image $x_i$ and its perturbed version $x_i'$, we used the Mean Squared Error: $\mathrm{MSE}(x_i, \ x_i') = \frac{1}{n} \sum_{j=1}^{n} \left( x_{i,j} - x_{i,j}' \right)^2$, where $x_{i,j}$ indicates the $j$-th pixel of $x_i$. MSE provides a quantitative measure of how much the perturbed image differs from the original at the pixel level.

\section{Localization Properties of Hierarchical Datasets}
\label{sec:mass_localization}

In this section, we investigate whether the class-conditional distributions associated with the hierarchical datasets introduced in Section~\ref{subsec:datasets} exhibit \emph{mass localization} in the sense of Pal et al.~\cite{pal2023avoidable}. In particular, we empirically assess whether, as input dimensionality increases, the probability mass of the various classes concentrates within subsets of exponentially small volume, as required by the $(a,b)$-localization condition.

To this aim, let $\DD_c = \{ x^c_i \in \DD \mid i = 1 \dots N_c \}$ be a dataset containing representative examples $x_i^c$ belonging to class $c$. To test whether the definition of $(a,b)$-localized distribution is satisfied, we estimate the volume of a bounding box (playing the role of the set $S$ in the definition of $(a,b)$-localization) containing all the samples in $\DD_c$. To avoid overestimating the volume of the bounding box, we align it with the centered principal component directions of $\DD_c$. Specifically, let $v^c_i$ denote the centered principal component coordinates of $x^c_i$. The volume of the resulting bounding box is
\begin{equation}
    \text{Vol}(Box_c) = \prod_{j = 1}^{n} \left[ \max\limits_{i=1 \dots N_c} v^c_{i,j} - \min\limits_{i=1 \dots N_c} v^c_{i,j} \right],
\label{eq.volBB}
\end{equation}
by letting
\begin{equation}
    \lambda_c = -\frac{\ln{\text{Vol}(Box_c) }}{n},
\end{equation}
we can say that $\DD_c$ is empirically $(a,b)$-localized with $a = \lambda_c$, $b = 0$ and $C = 1$\footnote{Considering the case $b = 0$ is a conservative estimate, as it provides an upper bound on the volume of the smallest set capturing most of the class probability mass.}.  

Motivated by the above arguments, we estimated $\lambda_c$ for the classes of the three hierarchical datasets VegSeed, Food-30, and ResynthDB. A practical challenge arises because estimating the volume of the $n$-dimensional bounding-box in PCA coordinates becomes ill-posed, with zero degenerate volume, when the sample covariance is rank-deficient. This limitation becomes critical at higher resolutions. For instance, for $96 \times 96$ images we have $n=27,648$, requiring at least an equivalent number of samples per class in order to obtain a full-rank covariance matrix, a requirement that far exceeds the size of the classes of the original datasets.
To circumvent this limitation, in this section we consider an additional reduced set of input resolutions, namely $32 \times 32$, $48 \times 48$, $64 \times 64$, and $96 \times 96$, and we apply an extensive data augmentation strategy in order to ensure that $N > n$.\footnote{The reduced resolutions used in this section are introduced solely to make the estimation of $Vol(Box_c)$ feasible in high dimension. They should not be confused with the main experimental resolutions (64, 128, 256, and 400) used in the subsequent sections to evaluate adversarial vulnerability.}
Specifically, we generated at least 30,000 images per class by applying a tailored number of semantic-preserving augmentations to each original sample (60 for Food-30, 100 for VegSeed, and 36 for ResynthDB). The augmentation pipeline \footnote{Available at \url{https://github.com/YaserGholizade/image_augmenter}} combines geometric and visual transformations, including flips, rotations, crop-and-zoom, filtering, and JPEG compression, to increase variability while preserving class identity.
It is important to remark that, in this context, augmentation is used primarily as a numerical device to mitigate rank-deficiency and enable stable volume estimation, rather than as an attempt to faithfully sample the underlying (unknown) class distribution.

Table~\ref{tab:bbox_rates} summarizes the results obtained by applying the above procedure to all classes and all considered resolutions. For each dataset and each resolution, we report the minimum, median, maximum, and average values of $\lambda_c$ across classes. As can be seen, $\lambda_c$ takes non-negligible values for all classes and all resolutions, providing strong empirical evidence in favor of mass localization. Moreover, $\lambda_c$ consistently increases with $n$, suggesting that the volume occupied by image classes decays at least exponentially fast with the ambient dimension, and possibly at a super-exponential rate. For instance, the average $\lambda_c$ for VegSeed rises from $0.80$ at $32\times32$ to $2.91$ at $96\times96$. 

To further evaluate the localization of classes with respect to the ambient space defined by all images in the dataset, we applied the same analysis to the entire dataset rather than to individual classes. However, computing the log volume over the full dataset is computationally prohibitive due to the size of the aggregated augmented data (e.g., approximately 450k images for VegSeed, 900k for Food-30, and 350k for ResynthDB). To address this issue, we implemented a bootstrapping approach to approximate the global PCA. Specifically, we randomly sampled subsets of 35,000 images from the full datasets, computed the log volume for each subset, and repeated this process 10 times. The reported value is the average log volume across these trials. We also examined the standard deviation across bootstrap runs (ranging
between 0.005 and 0.015) to confirm the stability of the metric and the absence of significant sampling effects. Even when considering the relative volume of class-specific bounding boxes with respect to the volume of the overall dataset bounding box, the localization property of the datasets is confirmed, with stronger localization for larger input sizes.

\begin{table}[ht]
\centering
\caption{Estimated $\lambda_c$ values for representative classes (Min/Median/Max volume), average class value, and the whole dataset across resolutions 32, 48, 64, and 96.}
\label{tab:bbox_rates}
\begin{tabular}{llcccc}
\toprule
\textbf{Dataset} & \textbf{Class / Type} & \textbf{32} & \textbf{48} & \textbf{64} & \textbf{96} \\
\midrule
\multirow{5}{*}{VegSeed}
 & Sponge Gourd      & 0.92 & 1.46 & 1.95 & 2.97 \\
 & Bottle Gourd   & 0.83 & 1.33 & 1.77 & 2.80 \\
 & Pumpkin           & 0.67 & 1.20 & 1.70 & 2.77 \\
 & \textbf{Average}        & \textbf{0.80} & \textbf{1.34} & \textbf{1.82} & \textbf{2.91} \\
 & Whole Dataset & 0.43 & 0.90 & 1.35 & 2.31 \\
\midrule
\multirow{5}{*}{Food-30} 
 & Beignets         & 1.01 & 1.40 & 1.75 & 2.44 \\
 & Macarons       & 0.79 & 1.18 & 1.55 & 2.29 \\
 & Pizza             & 0.66 & 0.99 & 1.34 & 2.09 \\
 & \textbf{Average}        & \textbf{0.82} & \textbf{1.17} & \textbf{1.53} & \textbf{2.27} \\
 & Whole Dataset  & 0.65 & 0.96 & 1.26 & 1.86 \\
\midrule
\multirow{5}{*}{ResynthDB} 
 & LeonardoAI       & 1.43 & 1.83 & 2.19 & 2.87 \\
 & Firefly        & 1.25 & 1.65 & 2.01 & 2.69 \\
 & StarryAI       & 1.18 & 1.59 & 1.96 & 2.66 \\
 & \textbf{Average}        & \textbf{1.28} & \textbf{1.66} & \textbf{2.02} & \textbf{2.70} \\
 & Whole Dataset  & 1.17 & 1.53 & 1.87 & 2.46 \\
\bottomrule
\end{tabular}
\end{table}

While Theorem~2.1 in~\cite{pal2023avoidable} already states that $(a,b)$-localization is a necessary condition for adversarial examples to be avoidable, it is instructive to interpret the estimates derived from our experiments in the context of the inevitability results of Shafahi et al.~\cite{shafahiadversarial}, summarized in Eq.~\eqref{eq.asrmse}.  
By assuming that the class distribution admits a density supported within the PCA-aligned bounding box, we obtain $U_c \ge 1/\text{Vol}(Box_c)$, and Eq.~\eqref{eq.asrmse} becomes
\begin{equation}
	P_{e,adv} \ge 1- \frac{1}{2\pi \sqrt{n \gamma}} e^{-n (\pi \gamma - \lambda_c)}.  
\label{eq.asrmse_BB}  
\end{equation}
Typical values of $\gamma$ (see Section~\ref{sec::experimental_results}) range from $10^{-4}$ to $10^{-6}$ (or even smaller for harder classification tasks such as source attribution). When coupled with the $\lambda_c$ values found in our experiments, these values render the right-hand side of Eq.~\eqref{eq.asrmse_BB} smaller than zero, thus resulting in a meaningless prediction.

\section{Emergence of Adversarial Examples vs Input Dimensionality: Experimental Analysis}
\label{sec::experimental_results}

Having established that natural image classes exhibit strong mass localization, the assumptions behind concentration-based inevitability results are no longer satisfied and theory alone becomes inconclusive. We therefore turn to an empirical investigation of what happens in practice when the input dimensionality increases. 

Clearly, evaluating robustness for \emph{any} possible classifier as predicted by theory is infeasible, for this reason, for each task (dataset) and each resolution, we attacked the full set of models described in Section~\ref{sec:allmodels}, resulting in 16 trained models per resolution. For each model, we generated adversarial examples using a PGD attack.
The attack hyperparameters, most notably the perturbation budget $\varepsilon$ and the step size $\alpha$, were tuned so as to obtain the smallest distortion required to successfully attack each input. In this way, for every model we obtained a per-sample estimate of the minimum MSE distortion required for a successful attack.
To summarize robustness in a way that is comparable across models and resolutions, we first fixed an ASR of 90\% and computed the minimum distortion required to achieve it. Specifically, for each model we collected the per-sample minimum MSE values and defined MSE$_{90}$ as the smallest MSE threshold such that at least 90\% of the samples could be successfully attacked with a distortion not exceeding this value.
Tables \ref{tab:mse_at_asr90_seed_untargeted}, \ref{tab:mse_at_asr90_food_untargeted}, and \ref{tab:mse_at_asr90_sourceattribution_untargeted} report the value of MSE$_{90}$ for all datasets, models and input size.
Based on the results in the tables, ResNet50 with adversarial training turns out to be the most robust model at all resolutions for the VegSeed and Food-30 datasets, with the exception of VegSeed at resolution 64 for which the most robust classifier is EfficientNet with adversarial training. 
The situation is slightly different for the ResynthDB task where the most robust model is MobileNet with adversarial training at resolutions 64 and 128, EfficientNet with adversarial training at 256 resolution, and CLIP+MLP with JPEG pre-processing at resolution 400. Noticeably, in all cases but ResynthDB at resolution 400, adversarial training yields the most robust models among the considered defenses. In contrast, the identity of the most robust architecture varies across tasks, suggesting that the choice of the most robust model remains strongly task-dependent.

\begin{table}
\centering
\caption{MSE$_{90}$ values for all the attacked models and defenses on the VegSeed dataset at various resolutions. The most robust model at each resolution is highlighted in bold.}
\label{tab:mse_at_asr90_seed_untargeted}
\begin{tabular}{lccccc}
\toprule
\textbf{Model} & \textbf{Defense} & \textbf{64} & \textbf{128} & \textbf{256} & \textbf{400} \\
\midrule
\textbf{MobileNet}      & -          & 1.61e-5 & 1.37e-5 & 2.29e-6 & 2.30e-6 \\
\textbf{ResNet50}       & -          & 4.80e-5 & 2.07e-5 & 1.89e-5 & 4.32e-5 \\
\textbf{EfficientNet}   & -          & 2.83e-5 & 1.29e-5 & 6.06e-6 & 2.85e-6 \\
\textbf{CLIP+MLP}       & -          & 1.74e-6 & 9.02e-7 & 3.66e-7 & 1.73e-7 \\
\midrule
\textbf{MobileNet}      & JPEG       & 1.97e-5 & 6.79e-5 & 1.11e-5 & 5.04e-6 \\
\textbf{ResNet50}       & JPEG       & 6.51e-5 & 5.14e-5 & 2.92e-5 & 2.18e-5 \\
\textbf{EfficientNet}   & JPEG       & 2.94e-5 & 1.45e-5 & 9.77e-6 & 6.63e-6 \\
\textbf{CLIP+MLP}       & JPEG       & 2.11e-6 & 1.05e-6 & 5.28e-7 & 3.23e-7 \\
\midrule
\textbf{MobileNet}      & Denoise    & 3.04e-5 & 2.37e-5 & 6.71e-6 & 5.76e-6 \\
\textbf{ResNet50}       & Denoise    & 7.11e-5 & 3.06e-5 & 2.37e-5 & 1.76e-5 \\
\textbf{EfficientNet}   & Denoise    & 5.94e-5 & 2.19e-5 & 1.21e-5 & 4.82e-6 \\
\textbf{CLIP+MLP}       & Denoise    & 1.97e-5 & 1.00e-5 & 3.64e-6 & 1.71e-6 \\
\midrule
\textbf{MobileNet}      & Adv Train  & 9.14e-5 & 3.02e-5 & 2.05e-5 & 1.84e-5 \\
\textbf{ResNet50}       & Adv Train  & 2.42e-4 & \textbf{1.00e-4} &\textbf{ 7.36e-5} & \textbf{4.32e-5} \\
\textbf{EfficientNet}   & Adv Train  & \textbf{6.67e-4} & 6.78e-5 & 4.69e-5 & 2.19e-5 \\
\textbf{CLIP+MLP}       & Adv Train  & 4.54e-6 & 2.48e-6 & 1.37e-6 & 7.99e-7 \\
\bottomrule
\end{tabular}
\end{table}

\begin{table}
\centering
\caption{MSE$_{90}$ values for all the attacked models and defenses on the Food-30 dataset at various resolutions. The most robust model at each resolution is highlighted in bold.}
\label{tab:mse_at_asr90_food_untargeted}
\begin{tabular}{lccccc}
\toprule
\textbf{Model} & \textbf{Defense} & \textbf{64} & \textbf{128} & \textbf{256} & \textbf{400} \\
\midrule
\textbf{MobileNet}      & -          & 1.25e-5 & 3.01e-6 & 9.45e-7 & 4.49e-7 \\
\textbf{ResNet50}       & -          & 1.35e-5 & 4.86e-6 & 2.08e-6 & 1.72e-6 \\
\textbf{EfficientNet}   & -          & 1.15e-5 & 4.08e-6 & 1.07e-6 & 8.45e-7 \\
\textbf{CLIP+MLP}       & -          & 3.99e-6 & 2.01e-6 & 1.11e-6 & 6.12e-7 \\
\midrule
\textbf{MobileNet}      & JPEG       & 1.68e-5 & 5.63e-6 & 1.44e-6 & 7.41e-7 \\
\textbf{ResNet50}       & JPEG       & 1.86e-5 & 7.83e-6 & 2.61e-6 & 1.21e-6 \\
\textbf{EfficientNet}   & JPEG       & 1.05e-5 & 3.69e-6 & 1.14e-6 & 6.37e-7 \\
\textbf{CLIP+MLP}       & JPEG       & 2.00e-6 & 8.83e-7 & 6.23e-7 & 4.05e-7 \\
\midrule
\textbf{MobileNet}      & Denoise    & 1.83e-5 & 5.79e-6 & 2.01e-6 & 1.05e-6 \\
\textbf{ResNet50}       & Denoise    & 1.92e-5 & 8.26e-6 & 3.91e-6 & 2.63e-6 \\
\textbf{EfficientNet}   & Denoise    & 1.53e-5 & 6.92e-6 & 2.06e-6 & 1.27e-6 \\
\textbf{CLIP+MLP}       & Denoise    & 6.04e-5 & 3.19e-5 & 1.84e-5 & 1.12e-5 \\
\midrule
\textbf{MobileNet}      & Adv Train  & 4.71e-5 & 1.76e-5 & 9.17e-6 & 3.33e-6 \\
\textbf{ResNet50}       & Adv Train  & \textbf{3.16e-4} & \textbf{1.49e-4} & \textbf{8.65e-5} & \textbf{5.68e-5} \\
\textbf{EfficientNet}   & Adv Train  & 3.90e-5 & 1.74e-5 & 4.04e-6 & 2.89e-6 \\
\textbf{CLIP+MLP}       & Adv Train  & 6.58e-6 & 4.49e-6 & 2.30e-6 & 1.33e-6 \\
\bottomrule
\end{tabular}
\end{table}

\begin{table}
\centering
\caption{MSE$_{90}$ values for all the attacked models and defenses on the ResynthDB dataset at various resolutions. The most robust model at each resolution is highlighted in bold.}
\label{tab:mse_at_asr90_sourceattribution_untargeted}
\begin{tabular}{lccccc}
\toprule
\textbf{Model} & \textbf{Defense} & \textbf{64} & \textbf{128} & \textbf{256} & \textbf{400} \\
\midrule
\textbf{MobileNet}      & -          & 7.50e-6 & 9.73e-7 & 8.10e-8 & 8.12e-8 \\
\textbf{ResNet50}       & -          & 6.93e-6 & 1.92e-6 & 8.21e-8 & 5.26e-8 \\
\textbf{EfficientNet}   & -          & 6.36e-6 & 3.52e-6 & 2.27e-6 & 7.58e-7 \\
\textbf{CLIP+MLP}       & -          & 1.29e-6 & 5.15e-7 & 1.27e-7 & 4.94e-8 \\
\midrule
\textbf{MobileNet}      & JPEG       & 1.83e-5 & 1.99e-5 & 6.87e-7 & 4.69e-7 \\
\textbf{ResNet50}       & JPEG       & 1.95e-5 & 4.56e-6 & 1.32e-6 & 3.94e-7 \\
\textbf{EfficientNet}   & JPEG       & 1.70e-5 & 1.15e-5 & 2.31e-6 & 8.19e-7 \\
\textbf{CLIP+MLP}       & JPEG       & 2.27e-5 & 1.35e-5 & 6.37e-6 & \textbf{3.08e-6} \\
\midrule
\textbf{MobileNet}      & Denoise    & 2.27e-5 & 1.62e-6 & 9.16e-7 & 3.54e-7 \\
\textbf{ResNet50}       & Denoise    & 9.39e-6 & 4.09e-6 & 2.89e-7 & 1.33e-7 \\
\textbf{EfficientNet}   & Denoise    & 8.27e-6 & 5.40e-6 & 2.57e-6 & 1.33e-6 \\
\textbf{CLIP+MLP}       & Denoise    & 2.10e-5 & 7.89e-6 & 2.26e-6 & 8.14e-7 \\
\midrule
\textbf{MobileNet}      & Adv Train  & \textbf{7.11e-5} & \textbf{3.75e-5} & 5.01e-6 & 2.08e-7 \\
\textbf{ResNet50}       & Adv Train  & 5.15e-5 & 8.13e-6 & 9.35e-7 & 2.19e-7 \\
\textbf{EfficientNet}   & Adv Train  & 3.46e-5 & 1.42e-5 & \textbf{6.48e-6} & 1.49e-6 \\
\textbf{CLIP+MLP}       & Adv Train  & 2.55e-5 & 1.15e-5 & 2.04e-6 & 9.76e-7 \\
\bottomrule
\end{tabular}
\end{table}

Now that we have identified the most robust model for each task and each input size, we can characterize how robustness scales with input dimensionality by comparing the MSE-ASR trade-off across resolutions. Specifically, Fig. \ref{fig::all_untargeted} reports the ASR for each resolution, as a function of the MSE threshold (equivalently, the MSE required to reach a given target ASR), thus allowing us to assess whether adversarial examples become easier to craft as $n$ increases for any desired operating point, and not only at ASR$=90\%$.
As can be seen, a clear and consistent trend emerges across all three datasets: as the input dimensionality increases, adversarial examples can be crafted with progressively smaller perturbations. This effect is not limited to the reference operating point used to select the most robust model (ASR$=90\%$), but is observed across the full range of attack success rates, as evidenced by the systematic left-shift of the ASR-MSE curves when the image size increases. Quantitatively, the MSE required to achieve ASR$=90\%$ decreases by factors ranging from approximately $6\times$ (Food-30) to over $20\times$ (ResynthDB) when increasing the resolution from $64\times64$ to $400\times400$. Notably, the magnitude of the effect depends on the dataset and is strongest for ResynthDB, suggesting that the dimensionality-driven increase in vulnerability may be amplified in settings where classification relies on weaker cues. 

While ASR quantifies the ease with which adversarial examples can be crafted, a direct comparison with theoretical results requires evaluating the adversarial risk, which combines both vulnerability to adversarial perturbations and baseline classification performance. From an empirical perspective, this corresponds to computing the adversarial accuracy AAcc defined in Eq.~\eqref{eq::aac_asr}, and then deriving the adversarial risk as $\texttt{Risk}_\veps = 1 - \mathrm{AAcc}.$
Under the CI definition of adversarial examples, this relationship is exact. Moreover, it provides a good approximation under the ER and PC definitions when the perturbation magnitude is small and the clean accuracy of the classifier is sufficiently high. 
By exploiting this relationship and the SAcc values reported in the plots of Fig.~\ref{fig::all_untargeted}, we obtain the corresponding adversarial risk curves shown in Fig. \ref{fig::risk_all_untargeted}. By inspecting Fig. \ref{fig::risk_all_untargeted}, we observe that although models operating on lower-dimensional inputs typically exhibit lower clean accuracy, they require substantially larger distortions to reach high adversarial risk. Conversely, at very small perturbation levels higher-dimensional models may retain higher overall accuracy due to their superior clean performance.
This behavior suggests that robustness and standard accuracy may evolve differently as the input dimension increases.

Taken together, our results provide a clear empirical answer to the central question motivating this section. Although the strong mass localization observed in Section~\ref{sec:mass_localization} invalidates the assumptions behind concentration-based inevitability theorems, in practice adversarial vulnerability increases systematically with input dimensionality. In particular, across datasets, architectures, and defenses, higher-resolution inputs consistently require smaller perturbations to achieve a fixed level of ASR. This finding establishes a robust empirical form of curse of dimensionality for adversarial examples, and motivates the second part of the paper, where we investigate whether a similar trend holds under the stronger requirement of targeted adversarial examples.
\begin{figure*}
    \centering
    \includegraphics[width=\textwidth]{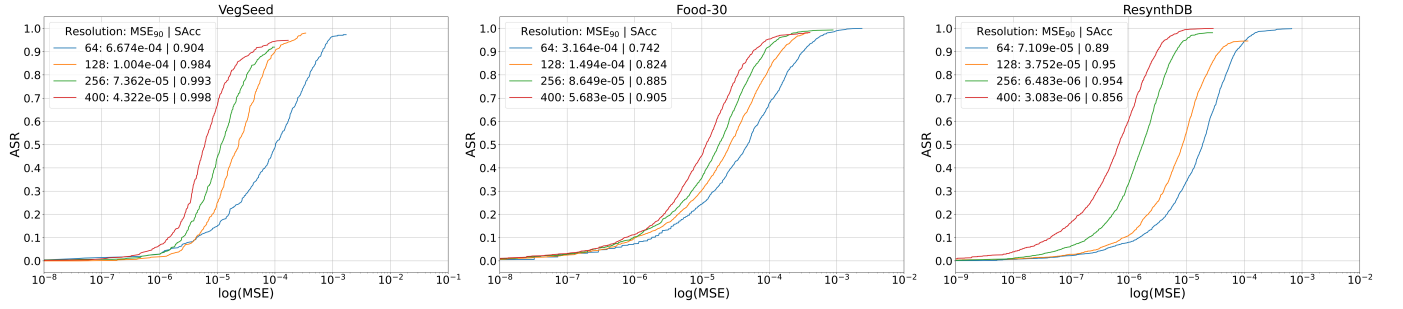}
    \caption{ASR vs. $\log(\text{MSE})$ in the untargeted adversarial attack scenario for the most robust models across datasets.
    Legend values report MSE$_{90}$ for each model.}
    \label{fig::all_untargeted}
\end{figure*}
\begin{figure*}
    \centering
    \includegraphics[width=\textwidth]{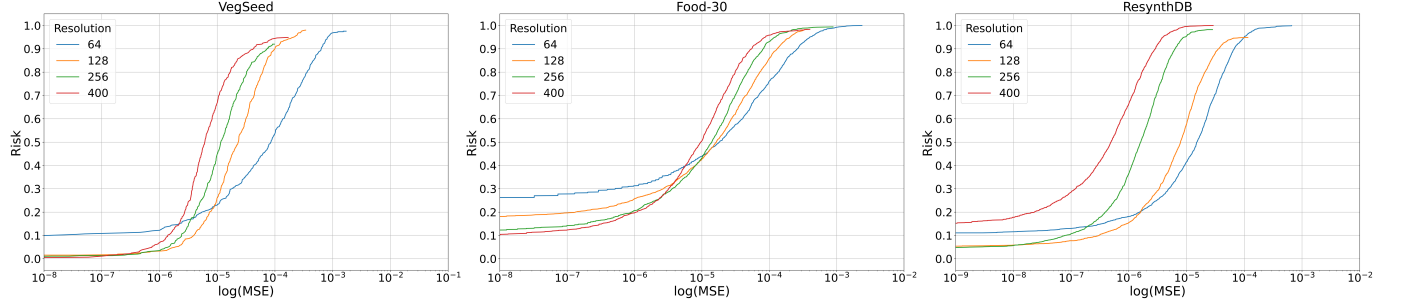}
    \caption{Risk vs. $\log(\text{MSE})$ in the untargeted adversarial attack scenario for the most robust models across datasets.}
    \label{fig::risk_all_untargeted}
\end{figure*} 
\section{Targeted Adversarial Examples: Formalization and Theoretical Analysis}
\label{sec:targeted}
We now investigate whether enforcing a specific target label requires significantly larger distortion than untargeted attacks.
\subsection{Targeted adversarial examples}
\label{subsec:targeted_def}
Let $\XX, \YY,$
$c$, and $f$ 
be defined as in Section \ref{sec.notanddef}. We denote by $\CC_i = \{x \in \XX: c(x) = i\}$, $i = 1, \dots, m$ a partition of $\XX$ into $m$ classes, and by $\FF_i = \{x \in \XX: f(x) = i\}$ the classification regions of the classifier $f$. 
We also let $V_i$ be the volume of $\FF_i$. A targeted adversarial example (in the CI sense) with source class $i$ and target class $j$, $i \ne j$, is a perturbed version of a sample $x$ belonging to the $i$-th  class, that is assigned to class $j$ by $f$, in formulas:
\begin{equation}
\label{eq:targeted_attack}
x^{i \rightarrow j} = x + \delta \quad \text{such that } x \in  \CC_i; ~ x^{i \rightarrow j} \in \FF_j,
\end{equation}
where $\delta$ denotes an imperceptible perturbation. Similar definitions can be given for the ER and PC cases. 
The extension of the notion of adversarial risk requires to specify over which source and target sets the probability that an adversarial example can be built is computed, yielding\footnote{Even in this case we adopt the CI definition of adversarial examples; similar definitions can be given in the ER and PC cases.}:

\begin{align}
\nonumber
&\texttt{Risk}^{i \rightarrow j}_\veps = \Pr_{x \sim \mu_i 
} \left\{ \exists x' \in \BB_\veps(x) \cap \XX : f(x') = j \right\}, i \ne j \\ 
&\texttt{Risk}^{j}_\veps = \frac{1}{m-1} \sum_{i : i \ne j} \texttt{Risk}^{i \rightarrow j}_\veps \\ \nonumber
&\texttt{Risk}^{t}_\veps = \frac{1}{m(m-1)} \sum_j \sum_{i : i \ne j} \texttt{Risk}^{i \rightarrow j}_\veps
\label{eq.targeteadvrisk}
\end{align}
where probabilities are computed with respect to the class distributions $\mu_i$.
In the first case, we consider the probability that targeted adversarial examples can be constructed for a specific source-target class pair. 
In the second case, we fix a target class and average over all possible source classes, thus measuring how easy it is to force the classifier toward a given target class regardless of the source class. 
In the third case, we measure the global targeted adversarial risk by averaging uniformly over all possible source-target class pairs.
Note that in the above definitions, we adopt uniform averaging over source and target classes in order to isolate the intrinsic geometric difficulty of targeted perturbations from class-frequency effects. For balanced datasets, this definition coincides with the standard risk averaged according to the data distribution.

\subsection{Targeted attacks and concentration of measure}
\label{subsec:theory extensions}
As discussed in Section \ref{sec.adv_and_conc}, the concentration of measure can be used to explain the emergence of untargeted adversarial examples for data following a concentrated distribution for which $\alpha(\veps)$ decreases fast enough with $n$. To the best of our knowledge no such result exists for the case of targeted attacks. In this section, we fill this gap by extending the analysis carried out in \cite{shafahiadversarial} to the case of targeted adversarial examples.   
The idea behind the extension is the following. We know from \cite{Ledoux1} (Proposition 2.8) that for the uniform measure on the hypercube $[0,1]^n$ equipped with the Euclidean distance, the concentration function satisfies:
\begin{equation}
\alpha(\veps) \le e^{-\pi \veps^2} = e^{-\pi n \gamma}.
\label{eq.alphacube}
\end{equation}
In \cite{shafahiadversarial}, such a property is exploited to prove that for any $\gamma$ and for any class $i$ such that $V_i < 1/2$, when $n$ tends to infinity the expansion of the complement of $\FF_i$ occupies most of the hypercube, and hence the volume of the samples that cannot be moved outside $\FF_i$ shrinks to zero. If the density function of class $i$ is uniformly bounded with respect to $n$, this implies that adversarial examples can be crafted with probability arbitrarily close to 1. The targeted case differs since the goal is to reach a specific region $\FF_j$, whose measure may be smaller than $1/2$. While concentration inequalities are strongest for sets of measure at least $1/2$, they can still be applied to smaller sets by adjusting the expansion radius. We show that this adjustment results in an additional distortion term which vanishes as $n \to \infty$, thereby proving that enforcing a specific target class incurs only a negligible asymptotic penalty compared to the untargeted case. A precise statement and proof of our result are given in the following.
\begin{theorem}{(Targeted adversarial examples)}
Let us consider a classification problem with input samples $x \in [0,1]^n$, ground-truth classes $\CC_i$, $i = 1,\dots,m$, and classification regions $\FF_j$, $j = 1,\dots,m$, associated with a classifier $f$. Let $V_j = \text{Vol}(\FF_j)$ be the volume of region $\FF_j$, and assume that a positive value $V_0$ independent of $n$ exists such that $V_j \ge V_0 > 0~\forall j$.  
Let $\mu_i$ denote the class-conditional distribution over $\CC_i$, and let $U_i < \infty$ be the supremum of $\mu_i$.
Then, given a sample $x \in \CC_i$, for any pair $i \neq j$, either with probability
\begin{equation}
P_{e,adv}^{i \rightarrow j} \ge 1 - U_i e^{-\pi n \gamma}
\label{eq.Pij}
\end{equation}
sample $x$ is misclassified as class $j$ ($x\in \FF_j$) or a targeted adversarial example $x^{i \rightarrow j}$ exists such that 
\begin{equation}
\frac{\|x - x^{i \rightarrow j} \|^2}{n} \le \gamma^{i \rightarrow j},
\label{eq.extradist}
\end{equation}
with
\begin{equation}
\lim_{n \rightarrow \infty} \gamma^{i \rightarrow j} = \gamma.
\end{equation}
\end{theorem}

\begin{proof}
Let us fix a target class $j$, and choose $\veps_0$ such that
\begin{equation}
\alpha(\veps_0) < V_0,
\label{eq.extraledoux}
\end{equation}
so that in particular $\alpha(\veps_0) < V_j$. 
By applying Lemma 1.1 in \cite{Ledoux1} to $\FF_j$ with the uniform measure over the hypercube, and exploiting the  concentration bound for the hypercube equipped with the Euclidean metric, and given that by hypotheses $\text{Vol}(\FF_j)=V_j \ge V_0>0$, we obtain
\begin{equation}
1 - \text{Vol}((\FF_j)_{\veps + \veps_0}) \le \alpha(\veps) \le e^{-\pi \varepsilon^2}.
\label{eq.lemma11a}
\end{equation}
Let $\mathcal{S}_j := [0,1]^n \setminus (\FF_j)_{\varepsilon+\varepsilon_0}$
denote the set of {\em safe} points whose distance from $\FF_j$ exceeds $\varepsilon+\varepsilon_0$. 
By \eqref{eq.lemma11a}, $\text{Vol}(\mathcal{S}_j)\le e^{-\pi \varepsilon^2}$. 
Since the class-conditional distribution $\mu_i$ is bounded by $U_i$, it follows that
\[
\mu_i(\mathcal{S}_j)\le U_i \cdot \text{Vol}(\mathcal{S}_j)\le U_i e^{-\pi \veps^2}.
\]
Recalling that $\varepsilon^2=n\gamma$, we obtain \eqref{eq.Pij}.
To go on, observe that in order for \eqref{eq.Pij} to hold,
the perturbation radius must be increased from $\varepsilon$ to
$\varepsilon + \varepsilon_0$, where $\varepsilon_0$ is chosen so that \eqref{eq.extraledoux} holds.
By the concentration bound for the uniform measure on the hypercube $
\alpha(\varepsilon_0) <
e^{-\pi \varepsilon_0^2}$,
implying that for \eqref{eq.Pij} to hold it is sufficient that
\begin{equation}
\varepsilon_0
>
\sqrt{\frac{1}{\pi}
\ln\!\left(\frac{1}{V_0}\right)}.
\label{eq.extrafinal}
\end{equation}

By the above and by recalling once more that $\veps^2 = n \gamma$, we conclude that in order to be able to craft a targeted adversarial example with probability at least $1 - U_i e^{-\pi n \gamma}$, it is sufficient that:
\begin{equation*}
\frac{(\veps + \veps_0)^2}{n} = \frac{\veps^2}{n} \left(1 + \frac{\veps_0}{\veps} \right)^2 > \gamma \Bigg( 1 + \sqrt{\frac{\ln\left(\frac{1}{V_0}\right)}{\pi n \gamma}} \Bigg)^2.
\label{eq.overall}
\end{equation*}
A valid choice for the bound in \eqref{eq.extradist} is then
\begin{equation*}
\gamma^{i \rightarrow j} = \gamma \Bigg( 1 + 2 \sqrt{\frac{\ln\left(\frac{1}{V_0}\right)}{\pi n \gamma}} \Bigg)^2,
\label{eq.boundex}
\end{equation*}
that tends to $\gamma$ when $n \rightarrow \infty$, thus completing the proof of the theorem.
\hfill$\square$
\end{proof}

The quantity $\gamma^{i \rightarrow j}$ upper bounds the distortion required to enforce a specific target class $j$ starting from samples of class $i$. The theorem shows that $\gamma^{i \rightarrow j}$ is asymptotically equivalent to the untargeted distortion $\gamma$, namely $\gamma^{i \rightarrow j}/\gamma \to 1$ as $n \to \infty$. Hence, under the stated assumptions, targeted attacks incur only a vanishing relative distortion overhead compared with untargeted attacks.

It is worth emphasizing that the result is established for the strongest notion of targeted adversarial risk, namely for each individual source–target pair $(i,j)$. 
Consequently, it holds \emph{a fortiori} for the averaged definitions of targeted risk, since averaging can only decrease the worst-case requirement.

With regard to the condition $V_j \ge V_0 > 0$ for all $j$, this is a reasonable assumption in typical classification settings under mild structural conditions. 
In particular, if the number of classes $m$ does not grow with the input dimension $n$, and if each class has non-vanishing prior probability and is represented by a non-degenerate region of the input space, then a classifier achieving good accuracy (in the absence of attacks) must assign a non-negligible portion of $[0,1]^n$ to each class. 
Since the decision regions form a partition of the hypercube with unitary total volume, their average volume is $1/m$, and it is natural to assume the existence of a dimension-independent lower bound $V_0$ such that $V_j \ge V_0$ for all $j$.

It is worth mentioning that while our analysis is carried out under the CI definition, 
in realistic high-dimensional settings, when the perturbation magnitude is small and the classifier achieves sufficiently high standard accuracy, the CI, ER, and PC definitions tend to be nearly equivalent, as discussed in Section~\ref{sec.notanddef}. In fact, extending the present theorem to the ER definition would require an additional structural assumption on the classifier, namely that for every pair $i \neq j$ the error regions $\CC_i \cap \FF_j$ are non-empty. 
Indeed, under the ER perspective, targeted adversarial examples can only be constructed if the classifier already misclassifies some samples from class $i$ as class $j$. 
While this condition is typically satisfied in practical large-scale classification systems with non-zero error rates across classes, it cannot be guaranteed in full generality without explicitly assuming the existence of such cross-class error regions. 

\section{Targeted Adversarial Examples: Empirical Analysis}
\label{sec:targeted_experimental_setup}
\begin{figure*}
    \centering
    \includegraphics[width=\textwidth]{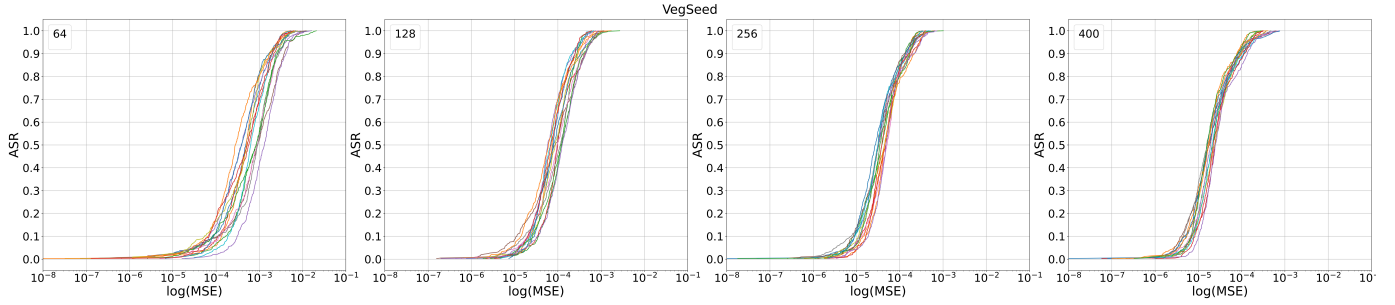}
    \caption{ASR vs. $\log(\text{MSE})$ in the targeted adversarial attack scenario for all target classesa and across all image resolutions. Plots refer to the most robust models trained on the VegSeed dataset. Each colour in the bundles refers to a different target class (ASR$_j$).}
    \label{fig::vegseed_targeted_curves}
\end{figure*}
In this section, we empirically analyze the behavior of targeted adversarial attacks as input dimensionality increases. Using the same hierarchical datasets, models, and attack methodology used for the untargeted case, we evaluate the distortion required to enforce specific target classes and compare it with the untargeted case. The goal of this analysis is to determine whether the additional control imposed by targeted attacks significantly increases the distortion required to craft adversarial examples and how such an increase depends on input dimensionality.
Specifically, we empirically assess how the adversarial risk $\texttt{Risk}^{j}_\veps$ for a given target class $j$ evolves when the input dimensionality grows. To do so, given a class $j$, we define a target-dependent subset of the dataset, denoted as $\NN_j$, containing all samples $x$ that are neither ground-truth members of class $j$ nor currently assigned to class $j$ by the classifier $f$, formally:
\begin{equation}
\label{eq::subset_targeted}
    \NN_j = \{ x \in \mathcal{D} \mid x \notin \CC_j \cup \FF_j \}.
\end{equation}
This set represents the pool of valid candidates for a targeted attack toward class $j$, ensuring that the attack is performed only on samples that neither belong to nor are already classified as the target class prior to perturbation.
The attack is performed for each subset $\NN_j$, by adopting the same evaluation setting described in Section \ref{sec:allmodels}, while fixing $\alpha=0.01$.
In particular, the attack on a sample $x \in \NN_j$ aims at constructing an adversarial example $x^{i \rightarrow j} \in \FF_j$ with the minimum perturbation.

With this setting, and for any maximum distortion, we computed the targeted attack success rate ASR$_j$ for the target class $j$ defined as

\begin{equation}
\mathrm{ASR}_{j} =
\frac{1}{|\NN_j|} \sum_{x_i \in \NN_j} \mathbf{1}
\left[ f(x_i') \in \FF_j \right],
\end{equation}

where, as usual, $x_i'$ is the output of the attack.

It is worth noting that ASR$_j$  differs slightly from $\texttt{Risk}^{j}_\veps$ since it does not account for the samples that are already misclassified by $f$ as class $j$. We decided to rely on ASR$_j$ for our analysis since it directly gives a measure of how successful the attack is regardless of the misclassification due to the non-ideality of the classifier, and because for accurate classifiers the fraction of samples already misclassified in the target class is typically negligible. 

The results we obtained for the VegSeed dataset are reported in Fig.~\ref{fig::vegseed_targeted_curves}\footnote{Similar results were obtained for the other datasets, but are not reported here due to space limitations.}, where the curves of ASR$_j$ are shown as a function of the MSE for all target classes.
Three main observations can be drawn from these results. First, for all target classes, the dependence of ASR$_j$ on the distortion closely follows the same monotonic behavior observed in the untargeted case, although with a target-dependent shift in the distortion required to achieve a given success rate. Second, the target class significantly affects the attack difficulty, with some targets requiring substantially larger perturbations than others. Third, as the input dimensionality increases, all curves systematically shift leftward, showing that targeted adversarial examples also become easier to craft in higher dimension.

To account for the dependence on the target class, Fig.~\ref{fig::all_targeted} reports the ASR averaged over all target classes, thus paralleling the definition of $\texttt{Risk}^{t}_\veps$. In addition, Fig.~\ref{fig::all_targeted_worst} reports a worst-case setting corresponding to the most difficult target class. In both cases, and across all datasets, the distortion required to achieve a given ASR decreases as the input dimension increases, confirming that the dimensionality effect observed for untargeted attacks persists in the targeted setting.

\begin{figure*}
    \centering
    \includegraphics[width=\textwidth]{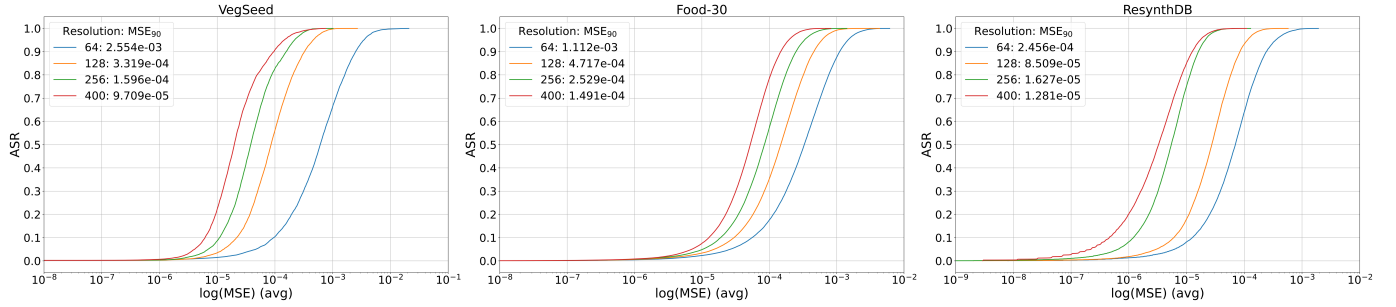}
    \caption{Average ASR vs. $\log(\text{MSE})$ in the targeted adversarial attack scenario for the most robust models across datasets. From left: VegSeed, Food-30, and ResynthDB dataset. Legend values indicate the MSE corresponding to an ASR of 90\% for each model.}
    \label{fig::all_targeted}
\end{figure*}
\begin{figure*}
    \centering
    \includegraphics[width=\textwidth]{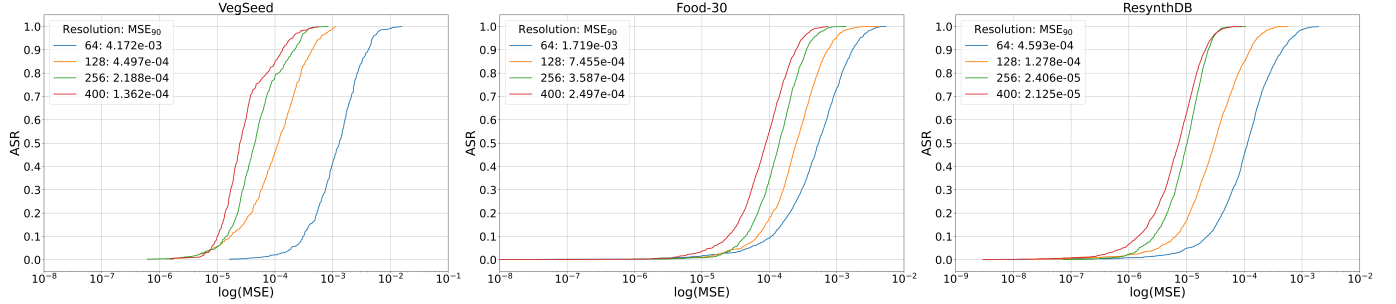}
    \caption{ASR vs. $\log(\text{MSE})$ in the worst-case targeted adversarial attack scenario for the most robust models across datasets. 
    Legend values indicate the MSE corresponding to an ASR of 90\% for each model.}
    \label{fig::all_targeted_worst}
\end{figure*}

We also carried out a detailed pairwise analysis by considering all source-target class pairs and computing either the distortion required to reach a fixed target success rate or the achieved ASR for a fixed distortion level. The corresponding results, obtained in the form of full source-target matrices, are not reported here due to space limitations. Nevertheless, they consistently confirm the same qualitative behavior observed above: the dependence on input dimensionality remains unchanged across all source-target pairs, and no source-target combination resulted in practically unattainable attacks.

Finally, we quantify the additional distortion required to move from untargeted to targeted attacks. Table~\ref{tab:merged_mse_full} reports the MSE corresponding to ASR$=90\%$ in the untargeted (U) and targeted (T) cases, together with the absolute difference (T$-$U) and the ratio T/U, respectively averaged across target classes and evaluated in the worst-case setting.
A first clear observation is that targeted attacks systematically require larger perturbations than untargeted ones, as expected from the additional control imposed on the attack outcome. At the same time, the absolute distortion gap T$-$U decreases consistently across all datasets and in both averaging settings as the input dimensionality increases. This behavior is in line with the theoretical result proved in Section~VII, according to which the additional distortion required to enforce a prescribed target class becomes asymptotically negligible.
In most cases, the ratio T/U also exhibits a decreasing trend as the input size grows, in agreement with the theoretical prediction that the relative overhead of targeted attacks should progressively vanish. The few non-monotonic behaviors occur mainly at the highest resolutions, particularly for the ResynthDB dataset, where the perturbations become extremely small and ratio estimates are correspondingly more sensitive to numerical fluctuations. Overall, these results indicate that, even in realistic finite-dimensional settings, the additional cost required to impose a specific target class remains limited and tends to become less significant when the input dimension increases.

\begin{table}
\centering
\caption{Targeted (T) vs Untargeted (U) MSE, on average and worst case scenarios.}
\label{tab:merged_mse_full}
    \begin{tabular}{lcccccc}
    \toprule
    \textbf{Dataset} & \textbf{Case} & \textbf{MSE} & \textbf{64} & \textbf{128} & \textbf{256} & \textbf{400} \\
    \midrule
        \multirow{7}{*}{VegSeed} & & U & 6.67e-4 & 1.00e-4 & 7.36e-5 & 4.32e-5 \\
        \cmidrule{2-7}
        & & T & 2.55e-3 & 3.32e-4 & 1.60e-4 & 9.71e-5 \\
        & Avg & T-U & 1.88e-3 & 2.32e-4 & 8.64e-5 & 5.39e-5 \\
        & & T/U& 3.82 & 3.32 & 2.17 & 2.25 \\
    
        \cmidrule{2-7}
        & & T & 4.17e-3 & 4.50e-4 & 2.19e-4 & 1.36e-4 \\
        & Worst & T-U & 3.51e-3 & 3.50e-4 & 1.45e-4 & 9.28e-5 \\
        & & T/U & 6.25 & 4.50 & 2.97 & 3.15 \\
    
    \midrule
        \multirow{7}{*}{Food-30} & & U & 3.15e-4 & 1.49e-4 & 8.65e-5 & 5.66e-5 \\
        \cmidrule{2-7}
        & & T & 1.12e-3 & 4.72e-4 & 2.53e-4 & 1.49e-4 \\
        & Avg & T-U & 8.05e-4 & 3.23e-4 & 1.67e-4 & 9.24e-5 \\
        & & T/U & 3.56 & 3.17 & 2.92 & 2.63 \\
    
        \cmidrule{2-7}
        & & T & 1.72e-3 & 7.46e-4 & 3.59e-4 & 2.50e-4 \\
        & Worst & T-U & 1.40e-3 & 5.97e-4 & 2.72e-4 & 1.93e-4 \\
        & & T/U & 5.46 & 5.00 & 4.15 & 4.41 \\
    
    \midrule
        \multirow{7}{*}{\shortstack[l]{ResynthDB}} & & U & 7.11e-5 & 3.75e-5 & 6.48e-6 & 3.28e-6 \\
        & & T & 2.46e-4 & 8.51e-5 & 1.63e-5 & 1.28e-5 \\
        & Avg & T-U & 1.75e-4 & 4.76e-5 & 9.82e-6 & 9.52e-6 \\
        & & T/U & 3.46 & 2.27 & 2.52 & 3.90 \\
        
        \cmidrule{2-7}
        & & T & 4.59e-4 & 1.28e-4 & 2.41e-5 & 2.13e-5 \\
        & Worst & T-U & 3.88e-4 & 9.03e-5 & 1.76e-5 & 1.80e-5 \\
        & & T/U & 6.46 & 3.41 & 3.71 & 6.48 \\
    \bottomrule
    \end{tabular}
\end{table}

\section{Conclusions}

In this paper we investigated the role of input dimensionality in the emergence of adversarial examples by combining theoretical analysis and large-scale experiments on hierarchical image datasets specifically designed to isolate the effect of image resolution. Our results show that, even if the assumptions underlying concentration-based impossibility results are violated in realistic image domains due to the strong localization of class distributions, adversarial examples become consistently easier to craft as the input dimension increases. This behavior was observed across multiple datasets, network architectures, and defense strategies, and holds for both untargeted and targeted attacks.

For targeted attacks, we  extended existing geometric concentration arguments, showing that under mild assumptions the additional distortion required to enforce a prescribed target class becomes asymptotically negligible. Experimental results confirm that, in practical settings, targeted attacks require only a limited extra distortion with respect to the untargeted case, and that this gap tends to shrink as dimensionality grows.

Taken together, these findings provide strong evidence of a curse of dimensionality underlying adversarial vulnerability. At the same time, they leave open a fundamental question: whether this phenomenon primarily originates from high-dimensional geometry itself, from the actual structure of real data distributions, or from specific properties of deep neural network classifiers.



\vspace{-0.3cm}
\bibliographystyle{IEEEtran}
\bibliography{advexamples.bib}

\end{document}